\documentclass[runningheads]{llncs}
\usepackage{graphicx}
\usepackage{amsmath,amssymb} 
\usepackage{xcolor}
\usepackage{soul}
\usepackage{amsmath}
\usepackage{amssymb}
\usepackage[latin1]{inputenc}   
\usepackage[T1]{fontenc}   
\usepackage{xspace}
\usepackage[vlined,ruled,linesnumbered]{algorithm2e}
\usepackage{subfig}
\usepackage{url}
\usepackage{bm} 
\usepackage{multirow}
\usepackage{multicol}
\usepackage{lipsum}                     
\usepackage{xargs}                      
\usepackage[colorinlistoftodos,prependcaption,textsize=small]{todonotes}
\usepackage{bbm}

\newcommand{\commentout}[1]{}

\def\ENet{Expert\xspace }
\def\ANet{Amateur\xspace }
\def\system{ExpertNet\xspace }
\def\TNet{TrustNet\xspace }
\def\SubNet{$\mathcal{F}(\cdot, \boldsymbol{\theta})$}
\def\SubNetk{$\mathcal{F}(\boldsymbol{x}_k, \boldsymbol{\theta})$}

\begin{document}
\title{TrustNet: Learning from Trusted Data Against (A)symmetric Label Noise}
%
%
\author{Amirmasoud Ghiassi\inst{1} \and
Taraneh Younesian\inst{1} \and
Robert Birke\inst{2} \and
Lydia Y.Chen\inst{1}}
\authorrunning{A.Ghiassi, et al.}
%
\institute{TU Delft, The Netherlands \\
\email{\{s.ghiassi,t.younesian\}@tudelft.nl} \\
\email{lydiaychen@ieee.org}\and
ABB Future Labs, Switzerland \\
\email{robert.birke@ch.abb.com}}

\maketitle              
\begin{abstract}
Robustness to label noise is a critical property for weakly-supervised classifiers trained on massive datasets. 
The related work on resilient deep networks tend to focus on a limited set of synthetic noise patterns, and with disparate views on their impacts, e.g., robustness against symmetric v.s. asymmetric noise patterns.
In this paper, we first derive analytical bound for any given noise patterns. Based on the insights, we design \TNet that first adversely learns the pattern of noise corruption, being it both symmetric or asymmetric, from a small set of trusted data. Then, \TNet is trained via a robust loss function, which weights the given labels against the inferred labels from the learned noise pattern.
The weight is adjusted based on model uncertainty across training epochs.
We evaluate \TNet on synthetic label noise for CIFAR-10 and CIFAR-100, and real-world data with label noise, i.e., Clothing1M. We compare against state-of-the-art methods demonstrating the strong robustness of \TNet under a diverse set of noise patterns. 

\end{abstract}
\section{Introduction}
Dirty data is a long standing challenge for machine learning models. The recent surge of self-generated data significantly aggravates the dirty data problems~\cite{blum2003noise,yan2014learning}. It is shown that data sets collected from the wild can contain corrupted labels as high as 40\%~\cite{xiao2015learning}. Even widely-adopted curated data sets, e.g., CIFAR-10, have incorrectly labeled images~\cite{chen:2019:ICML:understanding}. The high learning capacity of deep neural networks can memorize the pattern of correct data and, unfortunately, dirty data as well~\cite{arpit2017closer}. As a result, when training on data with non-negligible dirty labels~\cite{Zhang2017memorization}, the learning accuracy of deep neural networks can significantly drop.
 
While the prior art deems it imperative to derive robust neural networks that are resilient to label noise, there is a disparity in which noise patterns to consider and evaluate. The majority of robust deep networks against dirty labels focus on synthetic label noise, which can be symmetric or asymmetric. The former case~\cite{chen:2019:ICML:understanding} assumes noise labels can be corrupted into any other classes with equal probability, where the later case~\cite{wang2019symmetric} assumes only a particular set of classes are swapped, e.g., truck images are often mislabeled as automobile class in CIFAR-10. Patterns of noisy labels observed from real-life data sets, e.g., Clothing1M~\cite{xiao2015learning}, exhibit not only high percentages of label noise but also more complicated patterns mixing symmetric and asymmetric noises.  Moreover, there is a disagreement among related work on which noise patterns are more detrimental to regular networks and difficult to defend against~\cite{ma2018d2l,Vahdat:2017:NeurIPS:discriminative}. 

Noise patterns are commonly captured in transition matrices~\cite{chen:2019:ICML:understanding}, which describe the probability of how a true label is corrupted into another fake and observable label. A large body of prior art estimates such a labels transition matrix without knowing the true labels and incorporates such information into the learning process~\cite{patrini2017making}, particularly the computation of the loss function. Accurate estimation of the transition matrix can improve the robustness of neural networks, but it is extremely complicated when lacking the information on true labels and encountering sophisticated noise patterns~\cite{hendrycks2018using}. 

In contrast, adversarial learning~\cite{thekumparampil2018robustness,liu2019zk} advocates to train classification networks jointly with adversarial examples, i.e., corrupted labels. As such, the transition matrix can be conveniently learned with a sufficient number of adversarial examples and their ground truth. This is what we term \emph{trusted data} that contain not only given labels but also true labels validated by experts.
Hendrycks et. al~\cite{hendrycks2018using} shows that the resilience of deep networks can be greatly improved by leveraging such trusted data. The challenge is how to learn adversarial networks using a minimum set of trusted data that is difficult to obtain.

In this paper, we first develop a thorough understanding of the noise patterns, ranging from symmetric, asymmetric, and a mix of them. We extend the analysis from~\cite{chen:2019:ICML:understanding} and derive the analytical bound for classification accuracy for any given noise pattern. Our theoretical analysis compares real-world noise patterns against synthetic, symmetric, and simple asymmetric, noise.
Our findings on a diverse set of noise patterns lead us to focus on challenging cases where existing robust networks~\cite{patrini2017making,ma2018dimensionality} may fall short of defending against. 

The second contribution of this paper is a novel noise resilient deep network, namely \TNet, which leverages a small holdout of trusted data to estimate the noise transition matrix efficiently. Different from conventional adversarial learning, \TNet only tries to estimate the noise transition matrix, instead of learning the overall representation of adversarial data and hence requires only a small set of adversarial examples. Specifically, we first estimate the noise transition matrix through training \system~\cite{ExpertNet} on a small set of trusted data, i.e., 10\% of the training data. Such a trained \system can take images and their given labels as inputs and provide estimated labels -- additional label information. The core training step of \TNet is to weight the loss function from the given labels dynamically and inferred labels from \system. The specific weights are dynamically adjusted every epoch, based on the model confidence. We evaluate \TNet on CIFAR-10 and CIFAR-100, whose labels are corrupted by synthetically generated noise transition patterns. \TNet is able to achieve higher accuracy than SCL~\cite{wang2019symmetric}, D2L~\cite{ma2018dimensionality}, Boostrap~\cite{Reed:2015:ICLR:Bootstrap} and Forward~\cite{patrini2017making} in all most challenging scenarios. We also demonstrate the effectiveness of \TNet on a noisy real-world data set, i.e., Clothing1M, and also achieve higher accuracy. 
\section{Related work}\label{sec:related}
The problem of noisy labeled data has been addressed in several recent studies. 
We first summarize the impact of noise patterns, followed by the defense strategies that specifically leverage noise patterns.

\subsection{Impact of Noise Patterns}
Understanding the effect of label noise on the performance of the learning models is crucial to make them robust.
The impact of label noise in deep neural networks is first characterized~\cite{chen:2019:ICML:understanding} by the theoretical testing accuracy over a limited set of noise patterns. \cite{Vahdat:2017:NeurIPS:discriminative} suggests an undirected graphical model for modeling label noise in deep neural networks and indicates the symmetric noise to be more challenging than asymmetric. Having multiple untrusted data sources is studied by~\cite{Konstantinov:ICML19:Untrusted:dirtylabel} by considering label noise as one of the attributes of mistrust and assigning weights to different sources based on their reliability. However, it remains unclear how various kinds of noise patterns impact learning. 

\subsection{Resilient Networks Against (A)Symmetric Noise}

{\bf Symmetric Noise}
The following studies tackle the problem of symmetric label noise, meaning that corrupted labels can be any of the remaining classes with equal probability. 
One approach is to train the network based on noise resilient loss functions. D2L~\cite{ma2018d2l} monitors the changes in Local Intrinsic Dimension (LID) and incorporates LID into their loss function for the symmetric label noise. 
\cite{hendrycks2018using} introduces a loss correction technique and estimates a label corruption matrix for symmetric and asymmetric noise. 

Leveraging two different neural networks is another method to overcome label noise. Co-teaching~\cite{han2018co} trains two neural networks while crossing the samples with the smallest loss between the networks for both symmetric and asymmetric noise patterns.
Co-teaching$+$~\cite{yu2019does} focuses on updating by disagreement between the two networks on small-loss samples. 
\cite{Jiang:2018:ICML:MentorNet} combats uniform label flipping via a curriculum provided by the MentorNet for the StudentNet. However, these works do not explicitly model the noise pattern in their resilient models.

\noindent{\bf Asymmetric Noise}
Another stream of related work considers both symmetric and asymmetric noise.
One key idea is to differentiate clean and noisy samples by exploring their dissimilarity. 
~\cite{Han:2019:ICCV:selflearning,Lee:2018:CVPR:Cleannet} introduce class prototypes for each class and compare the samples with the prototypes to detect noisy and clean samples.
Decoupling~\cite{malach:2017:NIPS:decoupling} uses two neural networks and updates the networks when a disagreement happens between the networks. 

Estimation of the noise transition matrix is another line of research to overcome label noise. Masking~\cite{Han:2018:NeurIPS:Masking} uses human cognition to estimate noise and build a noise transition matrix. Forward~\cite{patrini2017making} corrects the deep neural network loss function while benefiting from a noise transition matrix. However, these studies do not consider the information in the noisy labels to estimate the matrix.

Building a robust loss function against label noise has been studied in the following works, although the dynamics of the learning model seem to be neglected. \cite{Zhang:2018:NeurIPS:GCEloss} provides a generalization of categorical cross entropy loss function for deep neural networks. The study~\cite{wang2019symmetric}, namely SCL, uses symmetric cross entropy as the loss function.  Bootstrapping~\cite{Reed:2015:ICLR:Bootstrap} combines perceptual consistency with the prediction objective by using a reconstruction loss. The research in~\cite{Goldberger:2017:ICLR:noise_layer,sukhbaatar:2014:training} changes the architecture of the neural network by adding a linear layer on top. 

In this work, we study both symmetric and various kinds of asymmetric label noise. 
We use the information in the noisy labels to 
estimate the noise transition matrix in an adversarial learning manner. Furthermore, we benefit from a dynamic update in our proposed loss function to tackle the label noise problem. 

\section{Understanding DNNs trained with noisy labels}\label{sec:noisepattern}
In this section, we present theoretical bounds on the test accuracy of deep neural networks assumed to have high learning capacity. Test accuracy is a common metric defined as the probability that the predicted label is equal to the given label. We extend prior art results~\cite{chen:2019:ICML:understanding} by deriving bounds for generic label noise distributions. We apply our formulation on three exemplary study cases and verify the theoretical bounds against experimental data. Finally, we compare bounds for different noise patterns providing insights on their difficulty for regular networks. 

\subsection{Preliminaries}
\label{subsec:preliminaries}
Consider the classification problem having dataset $\mathcal{D} = \{(\boldsymbol{x}_1, {y}_1),  (\boldsymbol{x}_2,  {y}_2), ... , (\boldsymbol{x}_N, {y}_N) \}$ where $\boldsymbol{x}_k$ denotes the $k^{th}$ observed sample, and ${y}_k \in C := \{0, ..., c-1\}$ the corresponding given class label over $c$ classes affected by label noise. Let $\mathcal{F}(\cdot, \boldsymbol{\theta})$ denote a neural network parameterized by $\boldsymbol{\theta}$, and $y^\mathcal{F}$ denote the predicted label of $\boldsymbol{x}$ given by the network $y^\mathcal{F} =\mathcal{F}(\boldsymbol{x}, \boldsymbol{\theta})$. The label corruption process is characterised by a transition matrix $T_{ij} = P(y = j | \hat{y} = i)$ where $\hat{y}$ is the true label. Synthetic noise patterns are expressed as a label corruption probability $\varepsilon$ plus a noise label distribution. For example, symmetric noise is defined by $\varepsilon$ describing the corruption probability, i.e. $T_{ii} = 1 -\varepsilon, \forall i \in C$, plus a uniform label distribution across the other labels, i.e. $T_{ij} = \frac{\varepsilon}{c-1}, \forall i\neq j \in C$. 



\subsection{New Test Accuracy Bounds}

To extend the previous bounds, we first consider the case where all classes are affected by the same noise ratio. We then further extend to the case where only a subset of classes is affected by noise.

{\bf All class noise}: All classes are affected by the same noise ratio $\varepsilon$, i.e., meaning only $1-\varepsilon$ percentage of given labels are the true labels.
\begin{lemma}\label{th:general}
For noise with fixed noise ratio $\varepsilon$ for all classes and any given label distribution $P(y = j), i \neq j$, the test accuracy is

\begin{equation} \label{eq:general}
\begin{split}
   P({y}^\mathcal{F} = y) = (1 - \varepsilon)^2 + \varepsilon^2 \sum_{j \neq i}^{C}{P^2(y = j)}
   \end{split}
\end{equation}
\end{lemma}
\begin{proof}
We have that $T_{ii} = 1 -\varepsilon, \forall i \in C$ since all classes are affected by the same noise ratio. Moreover, the probability of selecting noisy class labels is scaled by the noise ratio $T_{ij} = \varepsilon\;P(y=j), j\neq i \in C$. Now:
\begin{equation} \label{eq:generalproof}
\begin{split}
   P(y^\mathcal{F} = y)
      & = \sum_{i}^C{P(\hat{y} = i) P(y^\mathcal{F} = y | \hat{y} = i)} \\
      & = \sum_{i}^C{P(\hat{y} = i)}\sum_{j}^C{T^2_{ij}} \\
      & = \sum_{i}^C{P(\hat{y} = i)}[T^2_{ii} + \sum_{j \neq i}^{C}{T^2_{ij}}] \\
      & = \sum_{i}^C{P(\hat{y} = i)}[(1 - \varepsilon)^2 + \varepsilon^2 \sum_{j \neq i}^{C}{P^2(y = j)}].
   \end{split}
\end{equation}

Since $\sum_{i}^C{P(\hat{y} = i)} = 1$, we obtain Eq.~\ref{eq:general}. \hfill $\qed$
\end{proof}

{\bf Partial class noise}: in this pattern only a subset $S$ of class labels are affected by a noise ratio, whereas the set $U = C \; \backslash \; S$ is unaffected by any label noise.

\begin{lemma}\label{th:partial}
For partial class noise with equal class label probability, where $S$ is the set affected by noise with ratio $\varepsilon$ and $U$ is the set of unaffected labels, the test accuracy is
\begin{equation} \label{eq:partial}
\begin{split}
   P(y^\mathcal{F} = y) = \frac{|U|}{|C|} + \frac{|S|}{|C|}[{(1 - \varepsilon)}^2 + \varepsilon^2 \sum_{j \neq i}^{S}{P^2(y = j)}]
   \end{split}
\end{equation}
\end{lemma}
\begin{proof}
We have that for affected labels in $S$ the same noise transition definitions hold, i.e. $T_{ii} = 1 -\varepsilon, \forall i \in S$ and $T_{ij} = \varepsilon\;P(y=j), j\neq i \in S$. For unaffected labels we have that $\varepsilon = 0$ hence $T_{ii} = 1, \forall i \in U$ and $T_{ij} = 0, j\neq i \in U$. 
Moreover, $P(\hat{y} = i) = \frac{1}{|C|}$ assuming all class labels are equally probable. Now:

\begin{equation} \label{eq:partialproof}
\begin{split}
  P(y^f = y) = & \sum_{i}^C{P(\hat{y} = i) P(y^f = y | \hat{y} = i)}\\
    = & \sum_{i}^{|U|}{P(\hat{y} = i) P(y^f = y | \hat{y} = i)} +
        \sum_{i'}^{|S|}{P(\hat{y} = i') P(y^f = y | \hat{y} = i')} \\
    = & \sum_{i}^U{P(\hat{y} = i)}\sum_{j}^U{T^2_{ij}} +
        \sum_{i'}^S{P(\hat{y} = i')}\sum_{j'}^S{T^2_{i'j'}} \\
    = & \sum_{i}^U{P(\hat{y} = i)}[T^2_{ii} + \sum_{j \neq i}^{U}{T^2_{ij}}]
      + \sum_{i'}^S{P(\hat{y} = i')}[T^2_{i'i'} + \sum_{j' \neq i'}^{S}{T^2_{i'j'}}] \\
    = & \frac{1}{|C|}\sum_{i}^U[T^2_{ii} + \sum_{j \neq i}^{U}{T^2_{ij}}]
      + \frac{1}{|C|}\sum_{i'}^S[T^2_{i'i'} + \sum_{j' \neq i'}^{S}{T^2_{i'j'}}] \\
    = & \frac{1}{|C|}\sum_{i}^U 1
      + \frac{1}{|C|}\sum_{i'}^S[(1 -\varepsilon)^2 + \varepsilon^2 \sum_{j' \neq i'}^{S}{P^2(y = j')}] \\
    = & \frac{|U|}{|C|}
      + \frac{|S|}{|C|}[(1 -\varepsilon)^2 + \varepsilon^2 \sum_{j' \neq i'}^{S}{P^2(y = j')}] \\
\end{split}
\end{equation}
\hfill $\qed$
\end{proof}

\subsection{Validation of Theoretical Bounds}

We validate our new bounds on three study cases by applying our theoretical bounds to three different noise patterns for CIFAR-10 under different noise ratios and comparing the results against empirical accuracy results.

As first new noise pattern, we consider noisy class labels following a truncated normal distribution $\mathcal{N}^T(\mu, \sigma, a, b)$. This noise pattern is motivated by the targeted adversarial attacks~\cite{goodfellow:2015:explaining}. We scale $\mathcal{N}^T(\mu, \sigma, a, b)$ by the number of classes and center it around a target class $\tilde{c}$ by setting $\mu = \tilde{c}$ and use $\sigma$ to control how spread out the noise is. $a$ and $b$ simply define the class label boundaries, i.e. $a=0$ and $b=c-1$. To compute the bound, we estimate the empirical distribution at the different classes and apply Eq.~\ref{eq:general}. The second noise pattern extends our previous case. This distribution, referred in short as bimodal hereon, combines two truncated normal distributions. It has two peaks in $\mu_1$ and $\mu_2$ with two different shapes controlled by $\sigma_1$ and $\sigma_2$. The peaks are centered on two different target classes $\mu_1 = \tilde{c_1}$ and $\mu_2 = \tilde{c_2}$.
The third noise pattern considers partial targeted noise where only a subset of classes, $[2, 3, 4, 5, 9]$ in our example, are affected by targeted noise, i.e. swapped with a specific other class. Here we rely on Eq.~\ref{eq:partial} to compute the bound. This noise pattern has been studied in~\cite{wang2019symmetric}.

\begin{figure*}[t]
	\centering
	{
	\subfloat[Truncated normal, $\mu = 1, \sigma = 0.5$]{
	    \label{subfig:cm_normal}
	    \includegraphics[width=0.3\linewidth]{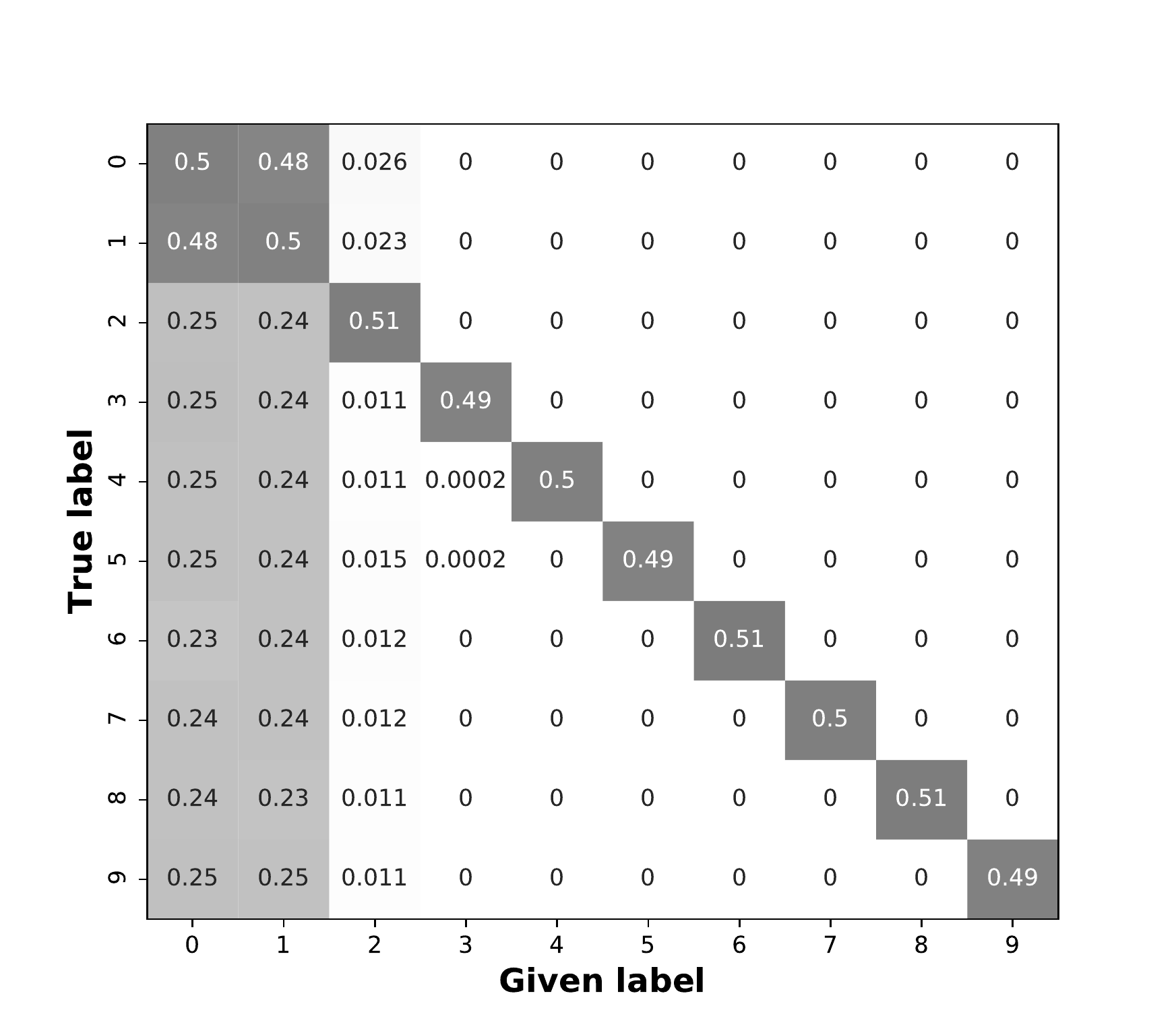}}
    \hfill
	\subfloat[Bimodal, $\mu_1 = 2,\sigma_1 = 1,\mu_2 = 7,\sigma_2 = 3$ ]{
	    \label{subfig:cm_bimodal}
	    \includegraphics[width=0.3\linewidth]{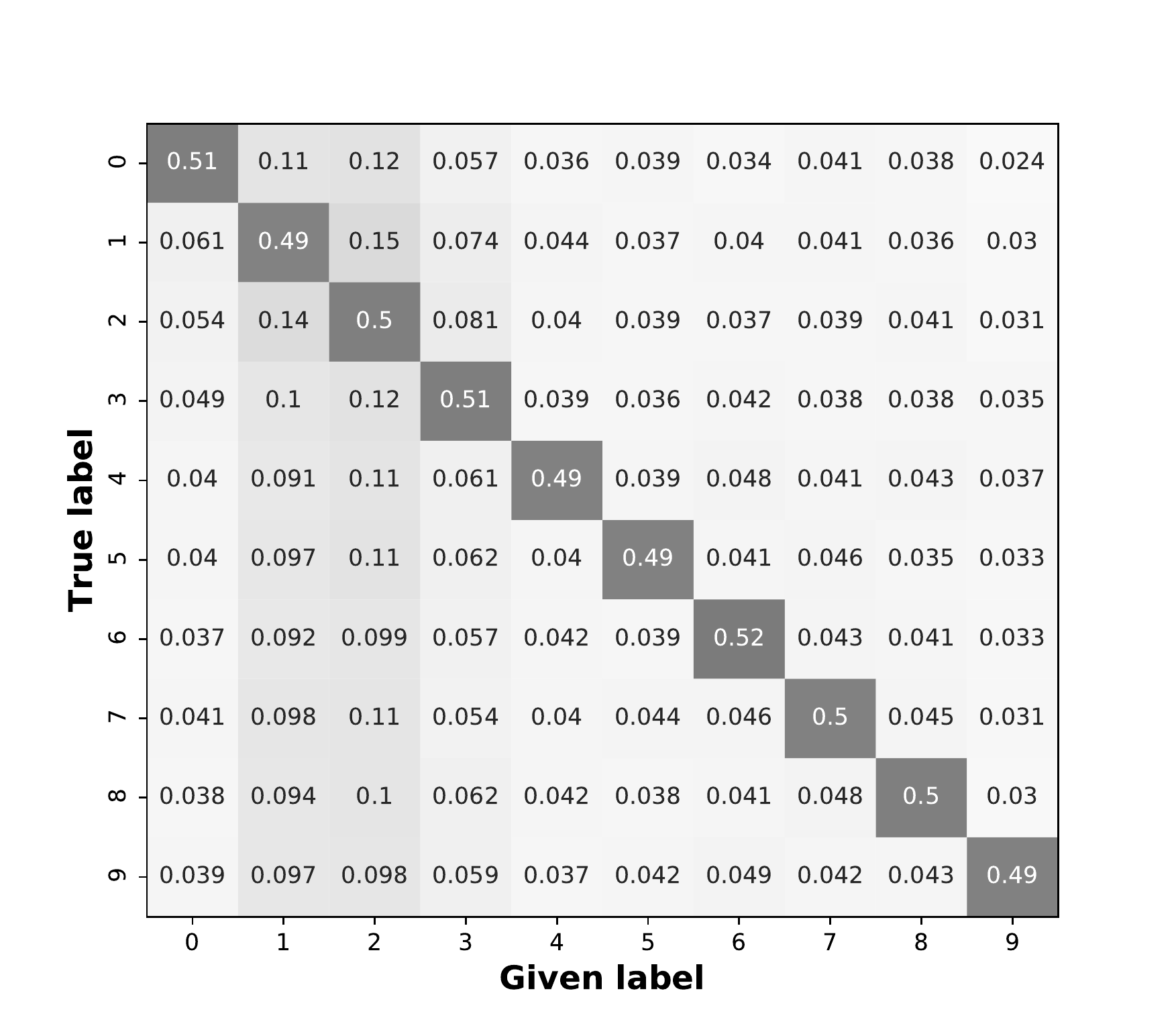}}
    \hfill
	\subfloat[Partial targeted]{
	    \label{subfig:cm_asymmetric}
	    \includegraphics[width=0.3\linewidth]{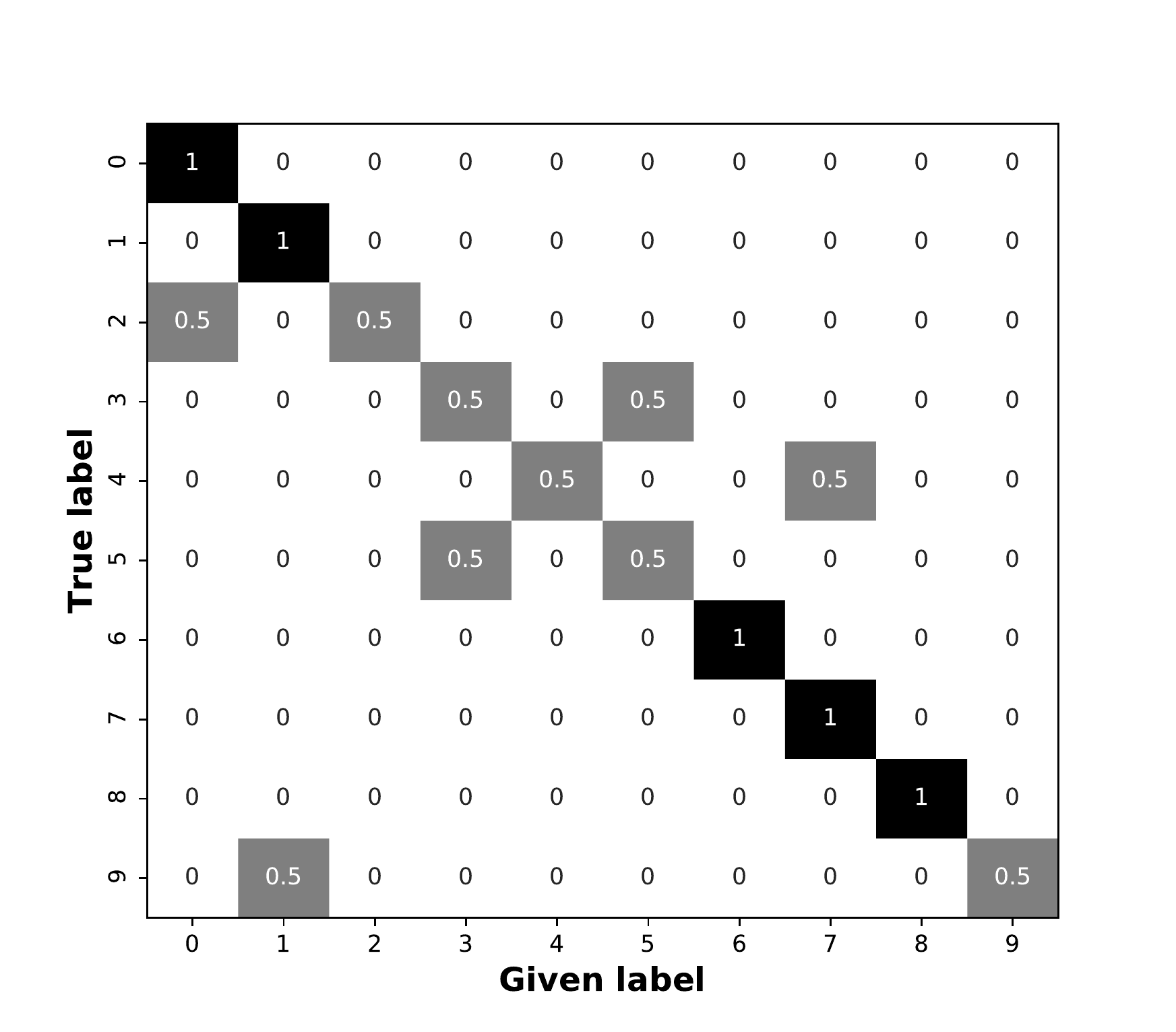}} \\
	    
	\subfloat[Truncated normal, $\mu = 1, \sigma = 0.5$]{
	    \label{subfig:acc_theoemp_tn}
	    \includegraphics[width=0.3\linewidth]{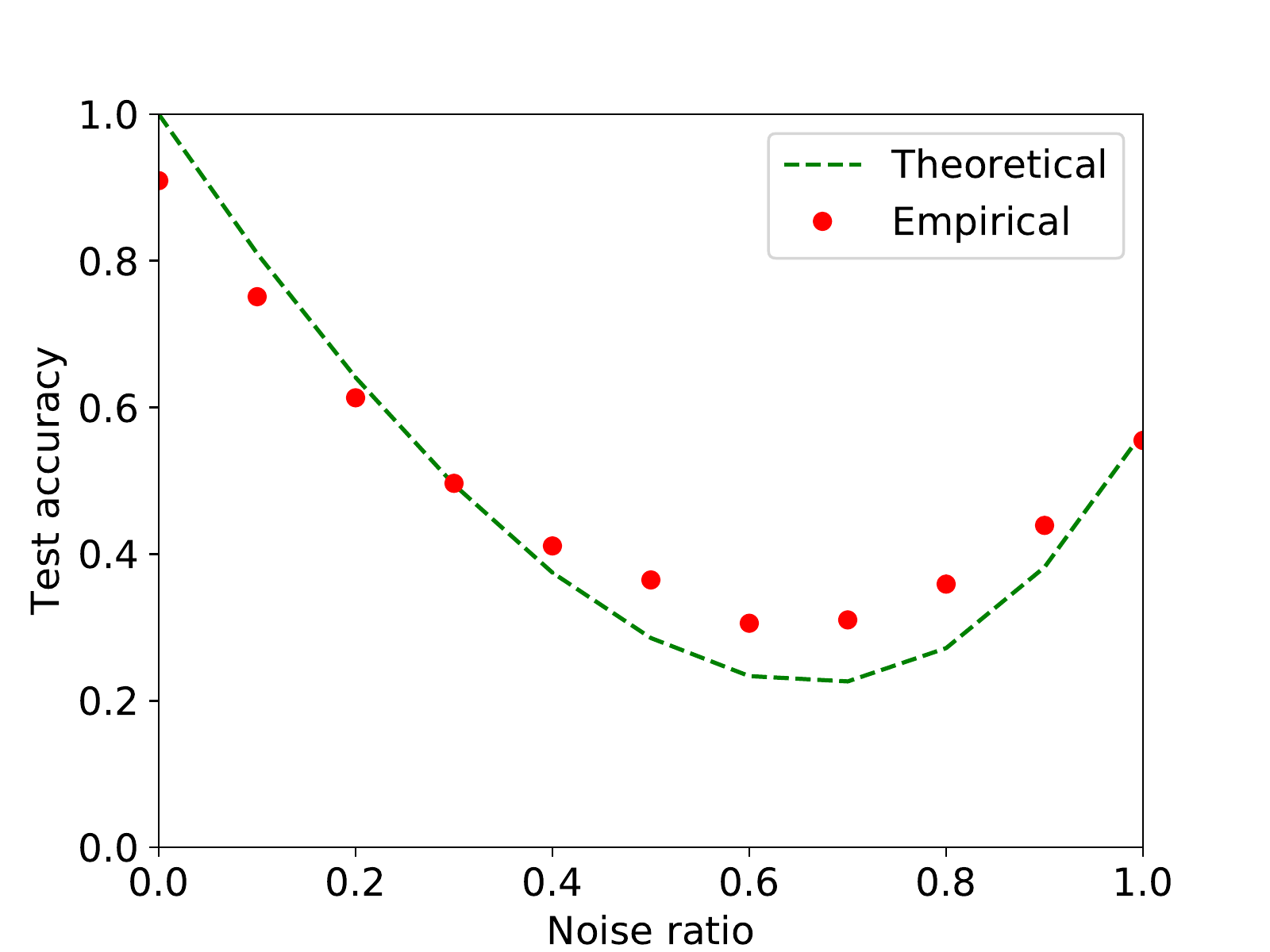}}
    \hfill
    \subfloat[Bimodal, $\mu_1 = 2,\sigma_1 = 0.5,\mu_2 = 7,\sigma_2 = 5$ ]{
	    \label{subfig:acc_theoemp_bi}
	    \includegraphics[width=0.3\linewidth]{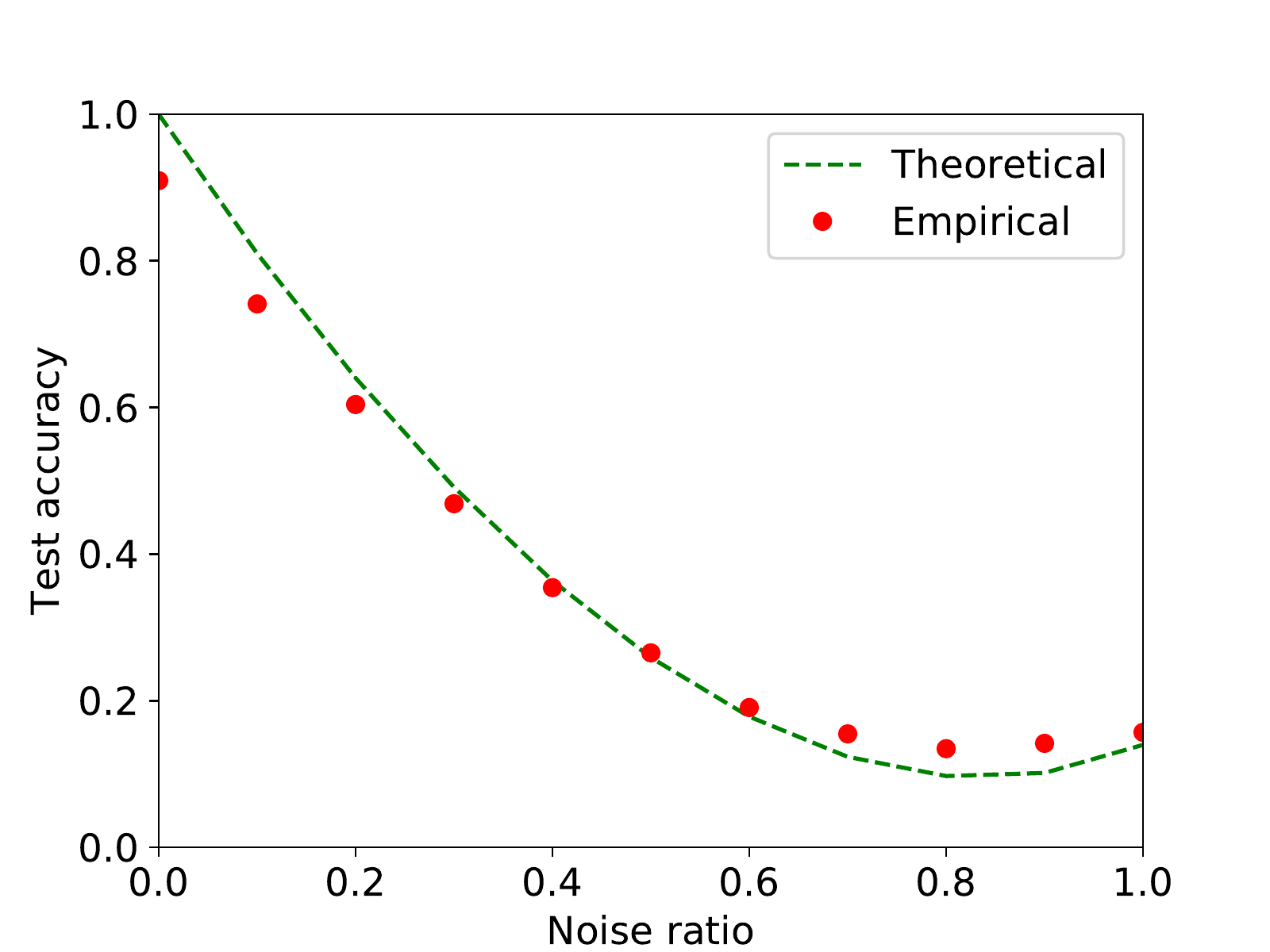}}
	\hfill
	\subfloat[Partial targeted]{
	    \label{subfig:acc_theoemp_pt}
	    \includegraphics[width=0.3\linewidth]{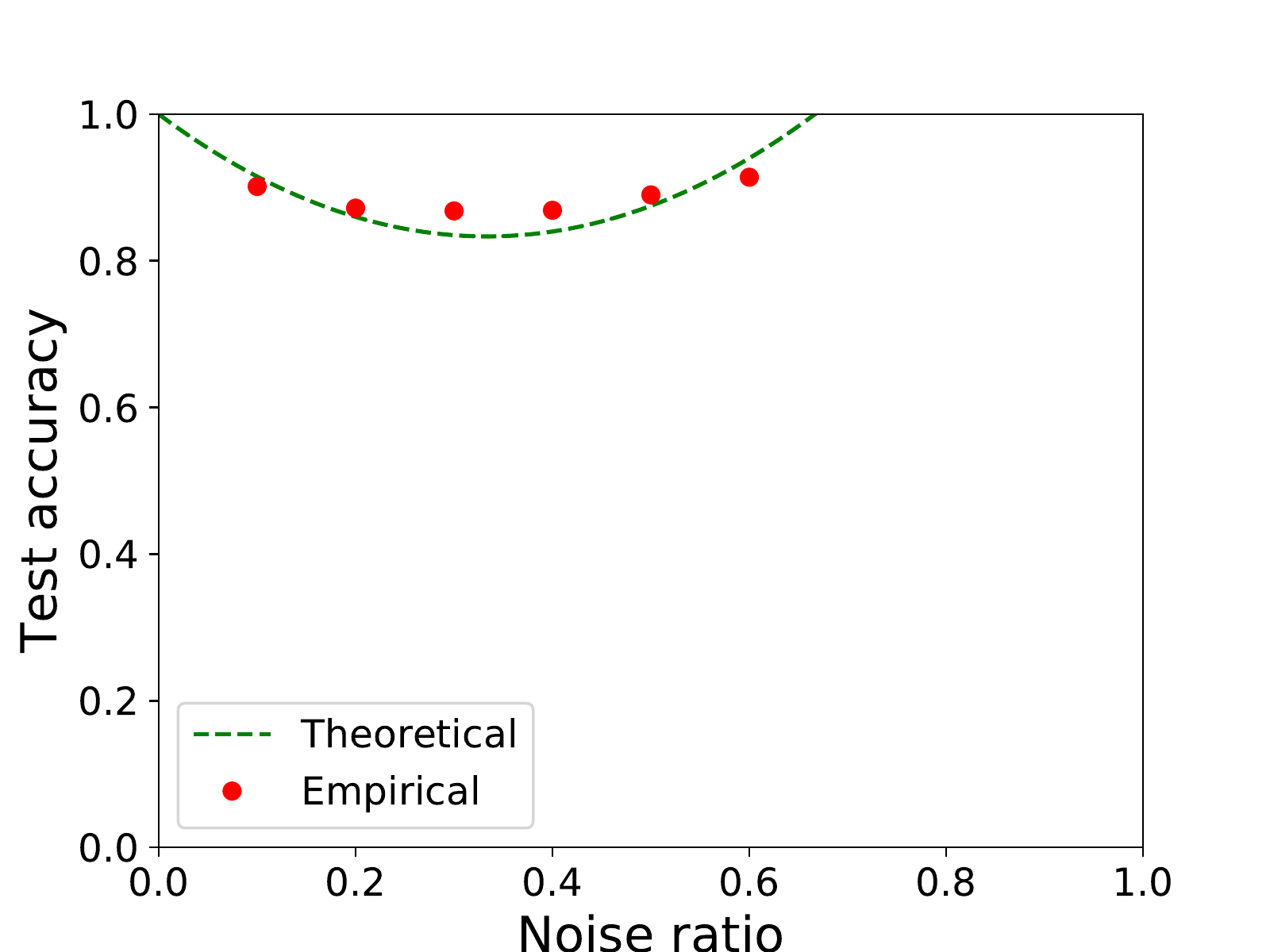}}    
	}
	\caption{Three study cases on CIFAR-10 with 10 classes. Top row shows the transition matrices for noise ratio $\varepsilon = 0.5$. Bottom row compares the theoretical bounds (lines) against empirical test accuracy results (points) for $0 \leq \varepsilon \leq 1$.}
	\label{fig:tmatrices}
\end{figure*}

Fig.~\ref{fig:tmatrices} summarizes the results. Top row shows the noise transition matrices for the three study noise patterns under noise ratio $\varepsilon = 0.5$. Bottom row compares the theoretical bounds against the empirical results obtained by corrupting the CIFAR-10 dataset with different noise ratios from clean to fully corrupted data: $0 \leq \varepsilon \leq 1$. The highest deviation between theoretical (lines) and empirical (points) results is shown for truncated normal noise around the deepest dip in the test accuracy, i.e., $\varepsilon = 0.7$. Here the theoretical accuracy result is 8.67\% points worse than the measured result. For the other two, the deviation is at most 4.06\% and 2.97\% (without considering $\varepsilon = 0.0$) for bimodal and partial targeted noise, respectively. Overall, the theoretical and empirical results match well across the whole range of noise ratios.

\subsection{Impact of Different Noise Patterns}


\begin{figure*}[t]
	\centering
	{
    \subfloat[Truncated normal: $\mu=1$, different $\sigma$]{
	    \label{subfig:diffvar}
	    \includegraphics[width=0.3\linewidth]{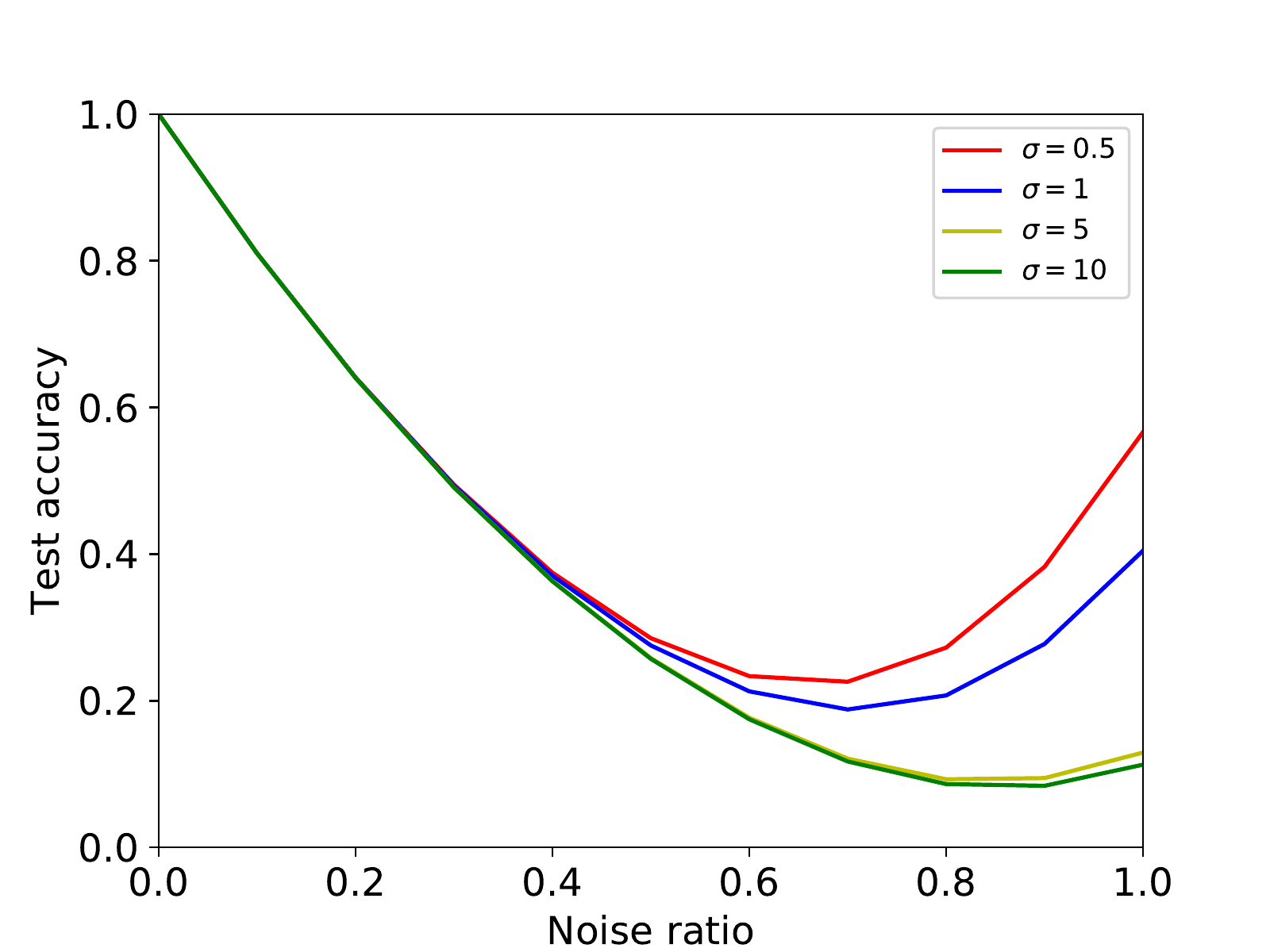}}
	 \hfill
	\subfloat[Bimodal: $\mu_1=2, \mu_2=7$, different $\sigma_1, \sigma_2$]{
	    \label{subfig:acc_types}
	    \includegraphics[width=0.3\linewidth]{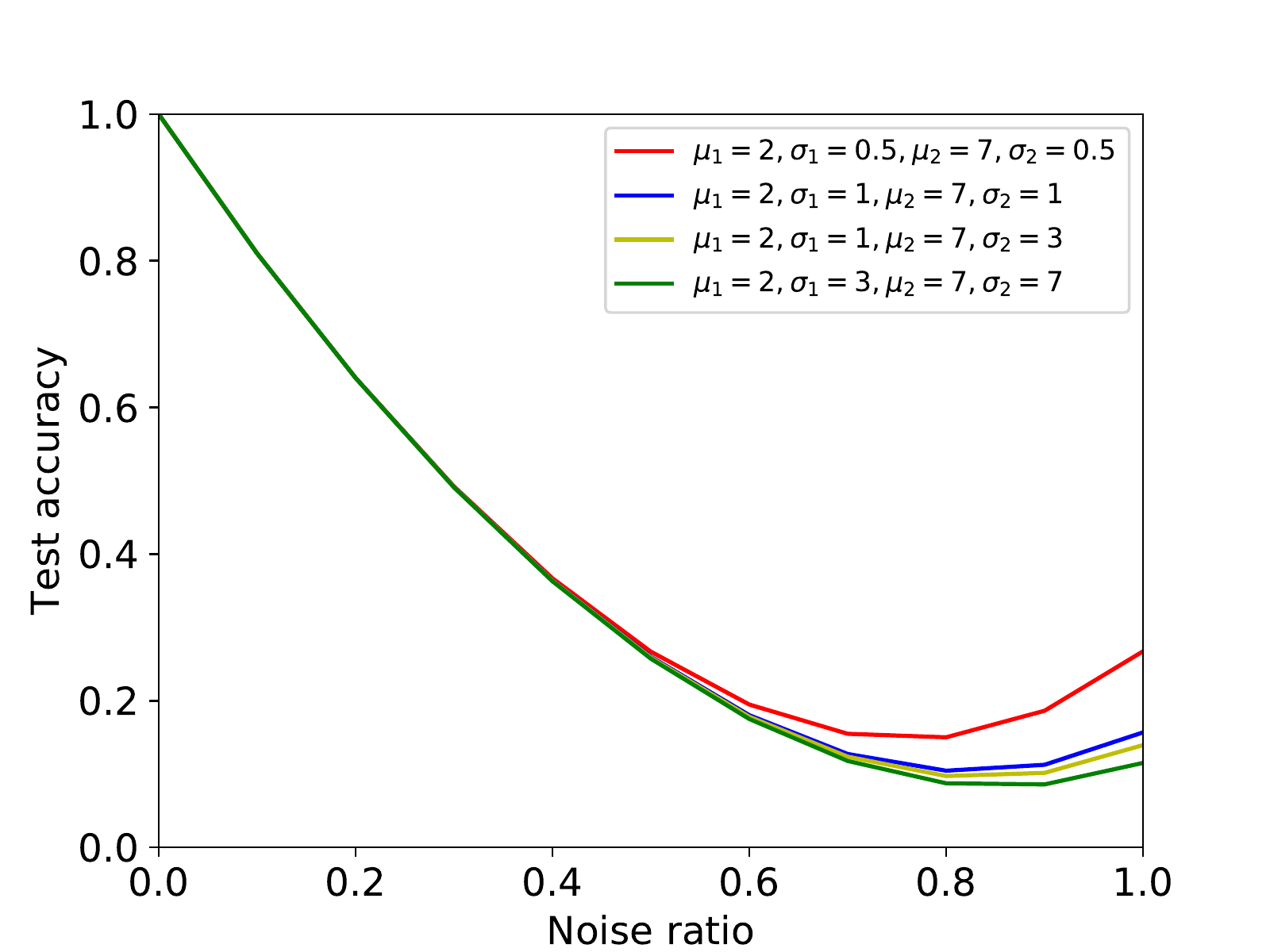}}
	\hfill
	\subfloat[All]{
	    \label{subfig:all}
	    \includegraphics[width=0.3\linewidth]{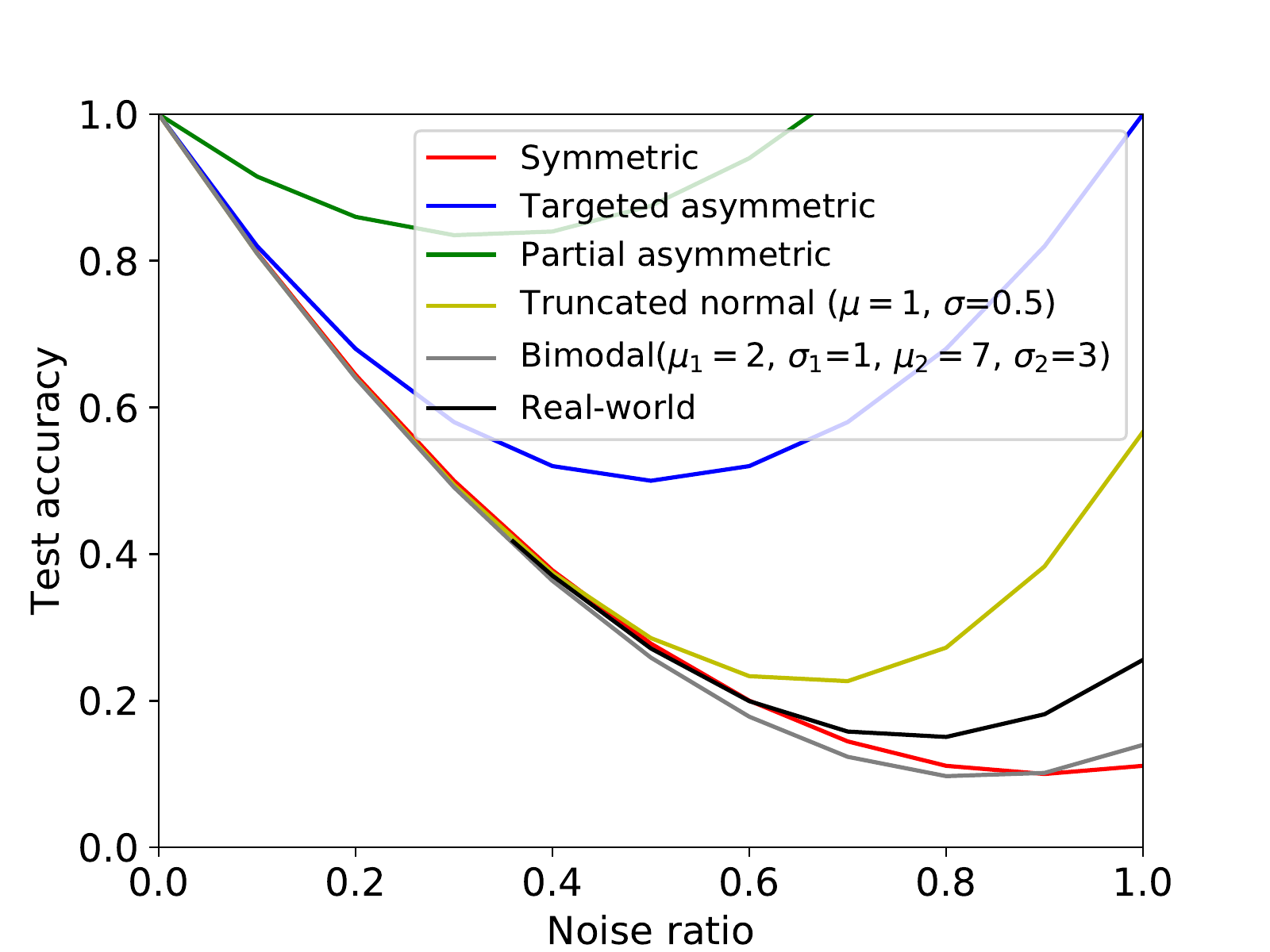}}
    }
	\caption{Analytical results: the impact of noise patterns with different parameters on the test accuracy across noise ratios.}
	\label{fig:theo_vs_emp}
\end{figure*}

We conclude by using our theoretical bounds to compare the impact on test accuracy of different noise patterns. First, we consider different parameters for truncated normal and bimodal noises and finish with comparing all noise patterns from here, in~\cite{chen:2019:ICML:understanding} and the real-world noise pattern from~\cite{xiao2015learning}. 

Fig.~\ref{fig:theo_vs_emp} shows all results. We start with truncated normal noise with a fixed target class and different $\sigma$. Higher values of $\sigma$ result into a wider spread of label noise across adjacent classes. Fig.~\ref{subfig:diffvar} shows the results. Under lower noise ratios, e.g., $\varepsilon < 0.5$, the impact of varying $\sigma$ is negligible, as shown by the overlapping curves. After that, we see that the most challenging cases are with high values of $\sigma$ due to the wider spread of corrupted labels deviating from their true classes.

Similarly to the previous analysis, for bimodal noise, we fix the target classes, i.e., $\mu_1$ and $\mu_2$, while varying the variances around the two peaks, i.e., $\sigma_1$ and $\sigma_2$. 
Overall the results are similar to truncated normal noise, but we can observe that the sensitivity to sigma is lower (see Fig.~\ref{subfig:acc_types}) even if on average test accuracy of truncated normal is higher than bimodal noise. For instance, in case of $\varepsilon = 1.0$ the difference between $\sigma = 0.5$ and $\sigma = 1$ is 16.26\% for truncated normal, but only 11.11\% for bimodal. Hence, bimodal tends to be more challenging since lines for different $\sigma$ are all more condensed around low values of accuracy with respect to truncated normal noise. 

To conclude, we compare all synthetic symmetric and asymmetric noise patterns considered against the real-world noise pattern observed on the {Clothing1M} dataset~\cite{xiao2015learning} (see Fig.~\ref{subfig:all}). The measured noise ratio of this dataset is $\varepsilon = 0.41$. To create the test accuracy bound, we scale the noise pattern to different $\varepsilon$ by redistributing the noise, such as to maintain all relative ratios between noise transition matrix elements per class. This imposes a lower limit on the noise ratio of $\varepsilon = 0.36$ to be able to keep all elements within the range $[0, 1]$. As intuition can suggest, partial targeted noise has the least impact since it only affects a fraction of classes. More interestingly, we see that the decrease in accuracy for all asymmetric noise patterns is not monotonic. 
When noise ratios are high, another class becomes dominant, and thus it is easier to counter the noise pattern.
On the contrary, all curves tend to overlap at smaller noise ratios, i.e., noise patterns play a weaker role compared to at higher noise ratios. Finally, the real-world noise pattern almost overlaps with bimodal. This might be due that errors in Clothing1M often are between two classes sharing visual patterns~\cite{xiao2015learning}. 

\section{Methodology}\label{sec:method}
In this section, we present our proposed robust learning framework, \TNet, featuring on a light weight estimation of noise patterns and a robust loss function.  

\subsection{\TNet Architecture}

\begin{figure}[t]
	\includegraphics[width=0.7\columnwidth]{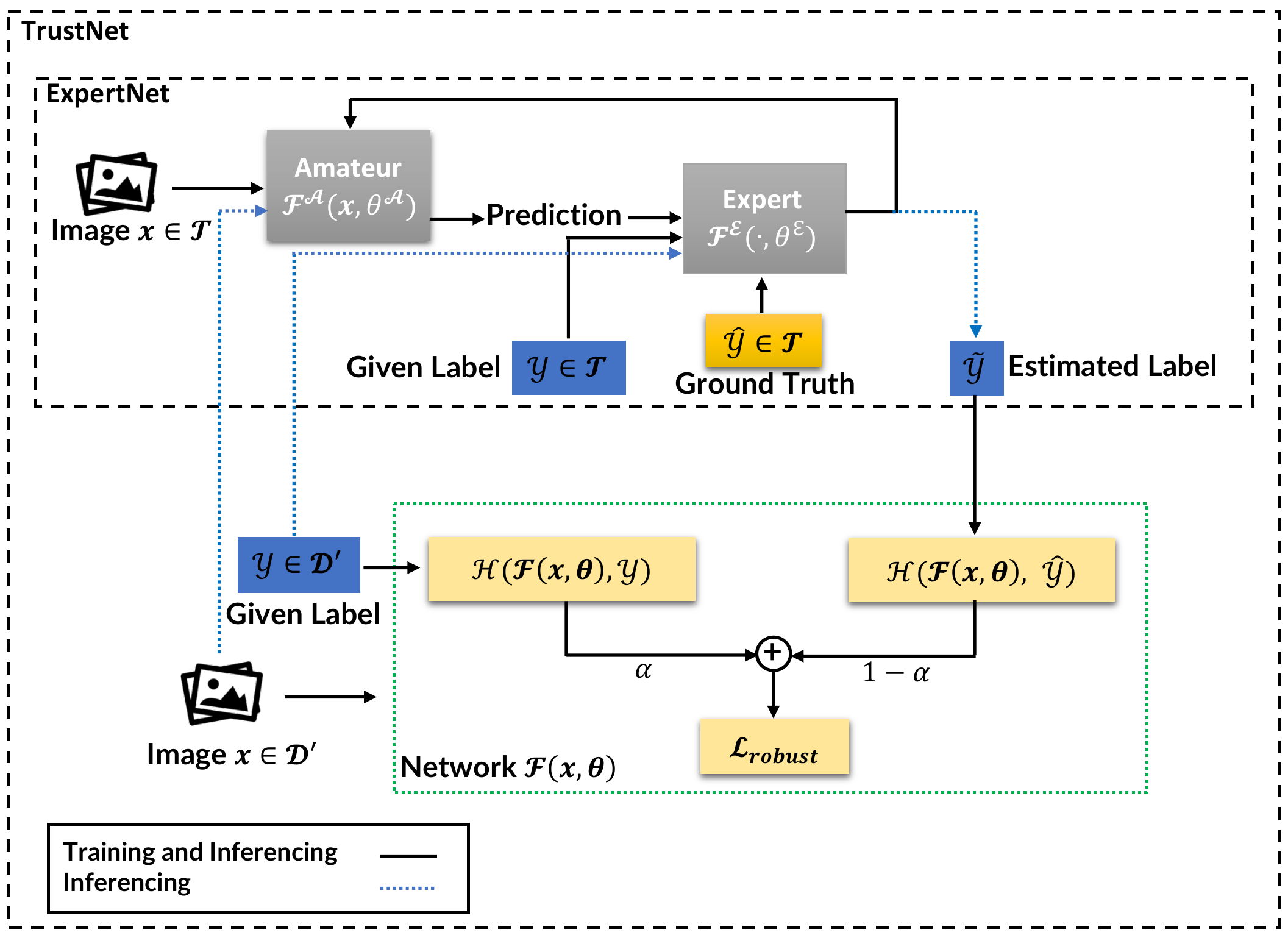}
	\centering
	\caption{\TNet architecture.}
	\label{fig:network}
\end{figure}

Consider extending the classification problem from Section~\ref{subsec:preliminaries} with a set of trusted data, $\mathcal{T} = \{(\boldsymbol{x}_1, {y}_1, \hat{y}_1),  (\boldsymbol{x}_2,  {y}_2, \hat{y}_1), ... , (\boldsymbol{x}_N, {y}_N, \hat{y}_N) \}$.
$\mathcal{T}$ is validated by experts and has for each sample $\boldsymbol{x}$ both given $y$ and true $\hat{y}$ class labels. Hence, our classification problem comprises two types of datasets: $\mathcal{T}$ and $\mathcal{D}$, where $\mathcal{D}$ has only the given class label $y$. The given class labels $y$ in both data sets are affected by the same noise pattern and noise ratio. Further, we assume that $\mathcal{T}$ is small compared to $\mathcal{D}$, i.e. $|\mathcal{T}| << |\mathcal{D}|$, due to the cost of experts' advise.


Corresponding to the two datasets, \TNet consists of two training routines highlighted by the top and bottom halves of Fig.~\ref{fig:network}. First (top half), \TNet leverages the trusted dataset to learn the underlying noise transition matrix via \system~\cite{ExpertNet}. \system is an adversarial network jointly learning from the given and true labels. Different from a fully fledged adversarial learning, \TNet only uses \system to learn the noise transition matrix instead of the representation of corrupted images. Second (bottom half), the trained \system is used to derive a dataset $\mathcal{D'}$ from $\mathcal{D}$ by enriching it with estimated class labels $\tilde{y}$ inferred by \system (blue path). Hence $\mathcal{D'} =  \{(\boldsymbol{x}_1, {y}_1, \tilde{y}_1),  (\boldsymbol{x}_2,  {y}_2, \tilde{y}_2), ... , (\boldsymbol{x}_N, {y}_N, \tilde{y}_N) \}$. Then, we train a deep neural network, \SubNet, on $\mathcal{D'}$ using the proposed robust loss function from Section~\ref{subsec:robustl}.
We note that the trusted data is used only to train \system, not \SubNet.

\subsection{Estimating Noise Transition Matrix}
\system is an adversarial network that consists of two neural networks: \ANet and \ENet. \ANet aims to classify images guided by the feedback from \ENet. \ENet acts as a supervisor who corrects the predictions of \ANet based on the ground truth.
Essentially, \ENet learns how to transform predicted labels to true labels, i.e., a reverse noise transition matrix.

During training, first \ANet provides for a sample $\boldsymbol{x}_k$ a prediction of the class probabilities $\boldsymbol{y}_k^\mathcal{A}$ to \ENet. \ENet uses $\boldsymbol{y}_k^\mathcal{A}$ concatenated with the given class label $y_k$ to learn to predict the ground truth class label $\hat{y_k}$. In turn, the predicted label from \ENet $y_k^\mathcal{E}$ is provided as feedback to train \ANet. In summary, training tries to minimize recursively the following two loss functions for \ANet, described by $\mathcal{F^\mathcal{A}}(\cdot, \boldsymbol{\theta}^\mathcal{A})$ and \ENet, described by $\mathcal{F^\mathcal{E}}(\cdot, \boldsymbol{\theta}^\mathcal{E})$:
\[
\min_{\boldsymbol{\theta}^\mathcal{A}} \mathcal{L}(\mathcal{F^\mathcal{A}}(\boldsymbol{x}_k, \boldsymbol{\theta}^\mathcal{A}), y_k^\mathcal{E})
\;\;\;\;\;\;\;\;  
\min_{\boldsymbol{\theta}^\mathcal{E}} \mathcal{L}(\mathcal{F^\mathcal{E}}(<\boldsymbol{y}_k^\mathcal{A}, y_k>,\mathcal{\boldsymbol{\theta}^\mathcal{E}}) , \hat{y}_k)
\]
where $<\cdot,\cdot>$ represents vector concatenation.~\cite{ExpertNet} provides the technical details.

The trained \system can estimate the true label from an image $\boldsymbol{x_k}$:
\begin{equation} \label{eq:enetpredict}
\tilde{y}_k = \mathcal{F^\mathcal{E}}(<\mathcal{F^\mathcal{A}}(\boldsymbol{x}_k, \boldsymbol{\theta}^\mathcal{A}),y_k>, \boldsymbol{\theta}^\mathcal{E}).
\end{equation}
Specifically, we use the trained \system to enrich and transform $\mathcal{D}$ in $\mathcal{D'}$ by incorporating for each image $\boldsymbol{x}_k$ the inferred class label $\tilde{y}_k$. Subsequently, we use  $\mathcal{D'}$ to train \SubNet\xspace via the loss function robust to noise from Section~\ref{subsec:robustl}. 


\subsection{Noise Robust Loss Function}
\label{subsec:robustl}

The given labels are corrupted by noise. Directly training on the given labels results in highly degraded performance as the neural network is not able to easily discern between clean and corrupted labels. To make the learning more robust to noise, \TNet proposes to modify the loss function to leverage both given labels $y$ and inferred labels $\tilde{y}$ from \system to train \SubNet.


The predicted label of \SubNet\xspace is compared, e.g., via cross-entropy loss, against both the given label and inferred label. The challenge is how to combine these two loss values. Ideally, for samples for which \system and \SubNet\xspace are highly accurate, the inferred label can be trusted more. On the contrary, for samples for which \system and \SubNet\xspace have low accuracy, the given labels can be trusted more.
Specifically, \TNet uses a weighted average between the loss of the predicted label from \SubNetk\xspace against both the given label $y_k$ and the \system 's inferred label $\tilde{y}_k$ with per sample weights $\alpha_k$ and $(1-\alpha_k)$ for all samples $\boldsymbol{x}_k$ in $\mathcal{D'}$.
Moreover, \TNet dynamically adjusts $\alpha_k$ after each epoch based on the observed learning performance of \SubNetk.  


In detail we use cross-entropy $H$ as standard loss measure to train our deep neural network \SubNetk:
\begin{equation} \label{eq:crossentropyloss}
    \mathcal{H}(\mathcal{F}(\boldsymbol{x}_k,\boldsymbol{\theta}) ,y_k) = - \sum_{i=0}^{c-1}\mathbbm{1}(y_k, c) \log \mathcal{F}(\boldsymbol{x}_k,\boldsymbol{\theta}) 
\end{equation}
where $\mathbbm{1}(y_k,c)$ is an indicator function equal to $1$ if $y_k = c$ and $0$ otherwise.
For each data point $\boldsymbol{x}_k$ in $\mathcal{D'}$, we assign weights of $\alpha_k$ and $(1-\alpha_k)$ to the cross-entropy of the given $y_k$ and inferred $\tilde{y}_k$ labels, respectively. We let ${\alpha_k} \in [0,1]$. Hence, we write the robust loss function $\mathcal{L}_{robust}$ as following: 
\begin{equation} \label{eq:trustloss}
\mathcal{L}_{robust}(\mathcal{F}(\boldsymbol{x}_k,\boldsymbol{\theta}), y_k, \tilde{y}_k) = \alpha_k \; \mathcal{H}(\mathcal{F}(\boldsymbol{x}_k, \boldsymbol{\theta}), y_k) + (1-\alpha_k) \; \mathcal{H}(\mathcal{F}(\boldsymbol{x}_k,\boldsymbol{\theta}), \tilde{y}_k).
\end{equation}
When the weight factor is low, we put more weight on the cross-entropy of inferred labels, and vice versa. In the following, we explain how to dynamically set $\alpha_k$ per epoch.

\subsubsection{Dynamic $\alpha_k$} 
Here we adjust $\alpha_k$ based on the uncertainty of \TNet and \system. When the learning capacities of \system and \TNet are higher (lower values of loss function), we have more confidence on the inferred labels and put more weight on the second term of Eq.~\ref{eq:trustloss}, i.e., smaller $\alpha_k$ values. As a rule of thumb, at the beginning $\alpha_k$ values are high since \TNet experiences higher losses at the start of training. Then $\alpha_k$ values gradually decrease with the growing capacity of \TNet.


Let $\alpha_{k,e}$ be the weight of the $k^{th}$ image at epoch $e$. We initialize $\alpha_{k,0}$ based on the entropy value $S$ from inferred class probabilities $\tilde{\boldsymbol{y}}_k$ of \system:
\[
S(\tilde{\boldsymbol{y}}_k) = - \sum_{i=0}^{c-1} \; \tilde{y}^i_k \; log \; \tilde{y}^i_k
\]
where $c$ is the number of classes and $\tilde{y}^i_k$ is the $i^{th}$ class probability of $\tilde{\boldsymbol{y}}_k$. We use \system since we do not have yet any predictions from \TNet's own neural network.

For subsequent epochs, $e>0$, we switch to \TNet as source of entropy values. We gradually adjust $\alpha_{k,e}$ based on the relative difference between current and previous epoch values:
\begin{equation} \label{eq:d_alpha}
\alpha_{k,e} = \alpha_{k,e-1}(1 + \frac{S(\boldsymbol{y}_k^\mathcal{F}(e))-S(\boldsymbol{y}_k^\mathcal{F}(e-1))} {S(\boldsymbol{y}_k^\mathcal{F}(e-1))}) \;\;\;\;\; \forall e>0,
\end{equation}
where $\boldsymbol{y}_k^\mathcal{F}(e)$ are the class probabilities predicted by \SubNet\xspace for the $k^{th}$ image at epoch $e$. When the entropy values decrease, we gain more confidence in \TNet and the weights on the inferred labels (1-$(1-\alpha))$ increase.

We summarize the training procedure of \TNet in Algorithm~\ref{alg:training}. Training \system consists of training two neural networks: \ENet , $\mathcal{F^\mathcal{E}}(\cdot, \boldsymbol{\theta}^\mathcal{E})$), and \ANet, $\mathcal{F}^\mathcal{A}(\cdot, \boldsymbol{\theta}^\mathcal{A})$, using the trusted data $\mathcal{T}$ for $E_{\system}$ epochs (line 1-4).
Then we need to compute the inferred labels for all data points in $\mathcal{D}$ to produce $\mathcal{D'}$ (line 5).
Finally, we train \TNet for $E_{\TNet}$ epochs (line 6-14). The initialization of $\alpha_k$ is via the entropy of the inferred labels (line 9) and then updated by the entropy of predicted labels (line 11). The robust loss function is computed accordingly (line 13).




\begin{algorithm}[tb]
\SetAlgoLined
\SetKwInOut{Input}{Input}\SetKwInOut{Output}{Output}
\Input{Trusted dataset $\mathcal{T}$, Untrusted dataset $\mathcal{D}$; Epochs $E_{\system}$,$E_{\TNet}$. \newline
Untrusted dataset $\mathcal{D}$ made of: Observed samples $\boldsymbol{x}$, Given labels $y$ \newline
Trusted dataset $\mathcal{T}$ made of: Observed samples $\boldsymbol{x}$, Given labels $y$, True labels $\hat{y}$ \\
}
\Output{Trained \TNet $\mathcal{F}(\boldsymbol{x},\boldsymbol{\theta})$}
Initialize $\mathcal{F}^\mathcal{A}$ and $\mathcal{F}^\mathcal{E}$ with random $\boldsymbol{\theta}^\mathcal{A}$ and $\boldsymbol{\theta}^\mathcal{E}$\\
\For{e = 0, 1, ..., $E_{\system}$ on $\mathcal{T}$ }{
    Train $\mathcal{F}^\mathcal{E}$ and $\mathcal{F}^\mathcal{A}$ \hfill \CommentSty{\#\system training}\\
    }

$\mathcal{D'} = \mathcal{D}$ extended with $\tilde{y} = \mathcal{F^\mathcal{E}}(<\mathcal{F^\mathcal{A}}(\boldsymbol{x}, \mathcal{\boldsymbol{\theta}^\mathcal{A}}),y>, \mathcal{\boldsymbol{\theta}^\mathcal{E}})$ \hfill  \CommentSty{\#\system inference}\\
Initialize $\mathcal{F}$ with random $\boldsymbol{\theta}$ \hfill
\CommentSty{\#\TNet training}\\
\For{e = 0, 1, ..., $E_{\TNet}$ on $\mathcal{D'}$ }{    
    \uIf{$e == 0$}{
    	$\alpha_{k,0} = S(\tilde{\boldsymbol{y}_k})$ 
    }
    \Else{
		$\alpha_{k,e} = \alpha_{k,e-1}(1 + \frac{S(\boldsymbol{y}_k^\mathcal{F}(e))-S(\boldsymbol{y}_k^\mathcal{F}(e-1))} {S(\boldsymbol{y}_k^\mathcal{F}(e-1))})$
    }
    Train $\mathcal{F}(\boldsymbol{x},\boldsymbol{\theta}_e)$ with
    $\alpha_{k,e} \; \mathcal{H}(\mathcal{F}(\boldsymbol{x}_k, \boldsymbol{\theta}_e), y_k) + (1-\alpha_{k,e}) \; \mathcal{H}(\mathcal{F}(\boldsymbol{x}_k, \boldsymbol{\theta}_e), \tilde{y}_k)$
    for each sample $k$ 
  }  
\caption{\TNet training}
\label{alg:training}
\end{algorithm}

\section{Evaluation}\label{sec:evalution}
In this section, we empirically compare \TNet against the state of the art noise, under both synthetic and real-world noises. We aim to show the effectiveness of \TNet via testing accuracy on diverse and challenging noise patterns.

\subsection{Experiments setup}
We consider three datasets: CIFAR-10~\cite{kriz-cifar}, CIFAR-100~\cite{kriz-cifar} and Clothing1M~\cite{xiao2015learning}.
CIFAR-10 and CIFAR-100 both have 60K images of $32 \times 32$-pixels organized in 10 and 100 classes, respectively. These two datasets have no or minimal label noise. We split the datasets into 50K training and 10K testing sets and inject into the training set the label noises from Section~\ref{sec:noisepattern}. We assume that 10\% of the training set forms the trusted data with access to the clean labels used as ground truth. We use this trusted set to learn the noise transition via \system. In turn, \system infers the estimated labels for the remaining training data. The whole training set is then used to train \TNet.
Clothing1M contains 1 million images scrapped from the Internet which we resize and crop to $224 \times 224$ pixels. Images are classified into $14$ class labels. These labels are affected by real-world noise stemming from the automatic labelling. Out of the 1 million images, a subset of trusted expert-validated images contains the ground truth labels. This subset consists of 47K and 10K images for training and testing, respectively. As for CIFAR-10 and CIFAR-100, we use the trusted set to train \system and infer the estimated labels for the rest of the dataset to train \TNet.
Note that for all three datasets, only training set is subject to label noise, not testing set.

The architecture of \ENet consists of a 4-layer feed-forward neural network with Leaky ReLU activation functions in the hidden layers and sigmoid in the last layer. This \ENet architecture is used across all datasets. \TNet and \ANet use the same architecture, which depends on the dataset. For CIFAR-10 \TNet and \ANet consist in an 8-layer CNN with 6 convolutional layers followed by 2 fully connected layers with ReLU activation functions as in~\cite{wang2018iterative}. For CIFAR-100 both rely on the ResNet44 architecture. Finally, Clothing1M uses pretrained ResNet101 with ImageNet. \TNet (\system) is trained for 120 (150) and 200 (180) for CIFAR-10 and CIFAR-100, respectively, using SGD optimizer with batch size 128, momentum $0.9$, weight decay $10^{-4}$, and learning rate $0.01$. Finally, Clothing1M uses 50 (35) epochs and batch size 32, momentum $0.9$, weight decay $5 \times 10^{-3}$ and learning rate $2 \times 10^{-3}$ divided by 10 every $5$ epochs.

Our target evaluation metric is the accuracy achieved on the clean testing set, i.e. not affected by noise. We compare \TNet against four noise resilient networks from the state of the art: SCL~\cite{wang2019symmetric}, D2L~\cite{ma2018dimensionality}, Forward~\cite{patrini2017making}, and Bootstrap~\cite{Reed:2015:ICLR:Bootstrap}. All training uses Keras v2.2.4 and Tensorflow v1.13.



\subsection{Synthetic Noise Patterns}
For CIFAR-10 and CIFAR-100, we inject synthetic noise. We focus on asymmetric noise patterns following a truncated normal and bimodal distribution, and symmetric noise, as discussed in Section~\ref{sec:noisepattern}. We inject noises with average rates $\varepsilon=0.4$, $0.5$ and $0.6$.
For truncated normal the target classes and variances are class 1 with $\sigma=0.5$ or $\sigma=5$ and 10 with $\sigma=1$ or $\sigma=10$ for CIFAR-10 and CIFAR-100, respectively.
For bimodal we use $\mu_1=2$, $\sigma_1=1$ plus $\mu_2=7$, $\sigma_2=3$ and $\mu_1=20$, $\sigma_1=10$ plus $\mu_2=70$, $\sigma_2=5$ for CIFAR-10 and CIFAR-100, respectively.

\subsubsection{CIFAR-10}
\begin{table}[t]
\centering
\caption{Accuracy on clean testing set for CIFAR-10 under 40\% and 60\% noise and patterns: i) symmetric, ii) bimodal with $\mu_1 = 2$, $\sigma_1 = 1$, $\mu_2 = 7$, $\sigma_2 = 3$, and iii) truncated normal with $\mu = 1$, $\sigma =[0.5, 5]$. Best results in bold.}
\label{tab_cifar10}
\resizebox{0.9\columnwidth}{!}{%
\begin{tabular}{|c|c|c|c|c|c|c|c|c|} 
\hline
\multirow{3}{*}{Methods}        & \multicolumn{2}{c|}{Symmetric}                                         & \multicolumn{2}{c|}{Bimodal Asymmetric}                                 & \multicolumn{4}{c|}{Truncated Normal Asymmetric}                                                         \\ 
\cline{2-9}
                                & \multirow{2}{*}{$\varepsilon = 0.4$} & \multirow{2}{*}{$\varepsilon = 0.6$} & \multirow{2}{*}{$\varepsilon = 0.4$} & \multirow{2}{*}{$\varepsilon = 0.6$} & \multicolumn{2}{c|}{$\varepsilon = 0.4$}         & \multicolumn{2}{c|}{$\varepsilon = 0.6$}          \\ 
\cline{6-9}
                                &                                    &                                   &                                    &                                    & $\sigma = 0.5$          & $\sigma = 5$             & $\sigma = 0.5$           & $\sigma = 5$             \\ 
\hline
TrustNet         & $77.03 \pm 0.32$  &  $61.22\pm 0.66$       & \pmb{$72.67\pm0.33$} & \pmb{$42.18\pm0.61$}    & $74.21 \pm 0.69$  &   $73.88 \pm 0.78$    &  $66.48 \pm 0.61$   & $67.23 \pm 0.57$    \\ 
\hline
SCL       & \pmb{$81.50\pm0.22$}  & \pmb{$73.13\pm0.12$} & $69.07\pm1.17$ &  $15.00\pm0.67$ & \pmb{$80.93\pm0.50$} & \pmb{$80.90\pm0.14$} & \pmb{$68.67\pm0.96$} & \pmb{$70.90\pm0.67$}     \\ 
\hline
D2L     & $75.87 \pm 0.33 $ &  $60.54\pm0.44$     &   $70.59. \pm 0.11$   &  $34.67 \pm 0.36$       &  $70.01 \pm 0.21$     & $71.22 \pm 0.57$    & $59.62 \pm 0.13$  & $62.35 \pm 0.43$                          \\ 
\hline
Forward  &$68.40\pm0.36$ & $51.27\pm1.11$   & $61.03\pm0.61$ & $33.27\pm0.53$ & $67.83\pm0.86$ & $68.63\pm0.65$ & $50.90\pm0.99$ & $51.53\pm0.74$     \\ 
\hline
Bootstrap   &$71.03\pm0.85$  & $56.47\pm1.18$ & $61.10\pm0.54$ & $31.17\pm0.59$ & $70.80\pm0.78$ & $71.07\pm0.78$ & $54.87\pm0.50$ & $55.80\pm1.23$   \\ 
\hline
\end{tabular}
}
\end{table}

We summarize the results of CIFAR-10 in Table~\ref{tab_cifar10}. We report the average and standard deviation across three runs. Overall the results are stable across different runs as seen from the low values of standard deviation. For readability reasons, we skip the results for 50\% noise in the table. These results follow the trend between 40\% and 60\% noise. 

\TNet achieves the highest accuracy for bimodal noises, which is one of the most difficult noise patterns based on Section~\ref{sec:noisepattern}. Here the accuracy of \TNet is consistently the best beating the second best method by increasing 2.4\%, 21.1\%, and 27.2\% for 40\%, 50\%, and 60\% noise ratios, respectively. At the same time, \TNet is the second best method for symmetric and truncated normal asymmetric noise. Here the best method is often SCL, which also leverages a modified loss function to enhance the per class accuracy using symmetric cross-entropy. This design targets direct symmetric noise where SCL outperforms \TNet. Considering the asymmetric truncated normal noise, the difference is smaller and decreasing with increasing noise ratio. At 60\% noise SCL is only marginally better by, on average, 2.9\%. Finally, test accuracy variations are not noticeable with increasing $\sigma$ values. All other baselines perform worse.

\subsubsection{CIFAR-100}

Table~\ref{tab_cifar100} summarizes the CIFAR-100 results over three runs. 
CIFAR-100 is more challenging than CIFAR-10 because it increases tenfold the number of classes while keeping the same amount of training data. This is clearly reflected in the accuracy results across all methods, but \TNet overall seems to be more resilient. Here, \TNet achieves the highest accuracy for both asymmetric noise patterns under all considered noise ratios. On average, the accuracy of \TNet is higher than SCL, the second best solution, by 2\%. The improvement is higher for higher noise ratios and lower variation, i.e., $\sigma = 1$. SCL outperforms \TNet on symmetric noise of low and middle intensity, i.e., $\varepsilon= [0.4, 0.5]$, but the difference diminishes with increasing noise, and at 60\% \TNet performs better. Different from CIFAR-10, test accuracy variations become noticeable for truncated normal noise with increasing $\sigma$ values producing a positive effect across most baselines. All other baselines perform worse.

\begin{table}[t]
\centering
\caption{Accuracy on clean testing set for CIFAR-100 under 40\% and 60\% noise and patterns: i) symmetric, ii) bimodal with  $\mu_1 = 20$, $\sigma_1 = 10$, $\mu_2 = 70$, $\sigma_2 = 5$, and iii) truncated normal with  $\mu = 10$, $\sigma =[1, 10]$. Best results in bold.}
\label{tab_cifar100}
\resizebox{0.9\columnwidth}{!}{%
\begin{tabular}{|c|c|c|c|c|c|c|c|c|} 
\hline
\multirow{3}{*}{Methods}        & \multicolumn{2}{c|}{Symmetric}                                         & \multicolumn{2}{c|}{Bimodal Asymmetric}                                 & \multicolumn{4}{c|}{Truncated Normal Asymmetric}                                                         \\ 
\cline{2-9}
                                & \multirow{2}{*}{$\varepsilon = 0.4$} & \multirow{2}{*}{$\varepsilon = 0.6$} & \multirow{2}{*}{$\varepsilon = 0.4$} & \multirow{2}{*}{$\varepsilon = 0.6$} & \multicolumn{2}{c|}{$\varepsilon = 0.4$}         & \multicolumn{2}{c|}{$\varepsilon = 0.6$}          \\ 
\cline{6-9}
                                &                                    &                                   &                                    &                                    & $\sigma = 1$          & $\sigma = 10$             & $\sigma = 1$           & $\sigma = 10$             \\ 
\hline
TrustNet     & $41.23 \pm 0.43$   &  \pmb{$29.11 \pm 0.12$}     & \pmb{$45.01 \pm 0.14$}& \pmb{$32.32 \pm 0.30$}       & \pmb{$37.66 \pm 0.36$}   &  \pmb{$44.56 \pm 0.42$}  & \pmb{$23.96 \pm 0.38$}   & \pmb{$33.29 \pm 0.41$}      \\ 
\hline
SCL         & \pmb{$42.30\pm0.36$} & $28.43\pm0.69$  & $43.57\pm0.42$ &  $30.70\pm0.88$  & $37.63\pm0.62$& $43.50\pm0.45$ &$19.20\pm0.57$ & $31.93\pm0.39$ \\ 
\hline
D2L    &$41.01 \pm 0.21$    &  $21.41 \pm 0.12$   & $32.47 \pm 0.43$   & $10.55 \pm 0.19$      & $10.66 \pm 0.16$     &  $10.32 \pm 0.21$ & $10.11 \pm 0.38$  & $10.05 \pm 0.14$ \\ 
\hline
Forward    & $36.40\pm0.37$ &  $16.00\pm0.80$      & $38.80\pm0.28$ &  $19.03\pm0.69$ & $34.03\pm0.33$ & $39.80\pm0.33$ & $10.27\pm0.47$ & $22.90\pm0.00$  \\ 
\hline
Bootstrap  & $28.40\pm0.16$ & $6.70\pm0.59$ & $32.17\pm0.62$ & $10.10\pm0.94$ & $27.23\pm0.71$ & $34.17\pm0.96$ & $6.10\pm0.16$ & $12.53\pm1.84$    \\ 
\hline
\end{tabular}
}
\end{table}

\subsection{Real-world Noisy Data: Clothing1M}

We use the noise pattern observed in real world data from the Clothing1M dataset to demonstrate the effectiveness and importance of estimating the noise transition matrix in \TNet. Table~\ref{tab:clothing1M} summarizes the results on the testing accuracy for \TNet and the four baselines. The measured average noise ratio across all classes is 41\%. Here, \TNet achieves the highest accuracy, followed by SCL and Forward. Forward is another approach trying to estimate the noise transition matrix. The better accuracy of \TNet is attributed to the additional label estimation from  \system learned via the trusted data and dynamically weighting the loss functions from given and inferred labels. The promising results here confirm that the novel learning algorithm of \TNet can tackle challenging label noise patterns appearing in real-world datasets.

\begin{table}[t]
\centering
\caption{Accuracy on clean testing set of real-world noisy Clothing1M.}
\label{tab:clothing1M}
\begin{tabular}{|c|c|c|c|c|c|} 
\hline
Methods  & TrustNet & SCL & D2L & Forward & Bootstrap  \\ 
\hline
Accuracy(\%) & {\bf 73.06}     & 70.78    & 69.43 & 70.04 & 68.77    \\
\hline
\end{tabular}
\end{table}

\section{Discussion}
In this section, we discuss testing accuracy on clean and noisy samples.
The bounds derived in Section~\ref{sec:noisepattern} consider testing on labels affected by the same noise as training data. This is due to the fact that the ground truth of labels is usually assumed unknown and not even available in the typical learning scenarios. However, the accuracy measured from the noisy testing data provides no information about how effective resilient networks defend the training process against the noisy data.    
Hence, related work on noisy label learning tests on clean samples, which show different trends as hinted in the evaluation section. Fig.~\ref{fig:empcleannoise} compares the two approaches across different noise patterns empirically. In general, in the case of clean test labels, the testing accuracy decreases with increasing noise ratios almost linearly. As for noisy labels, testing accuracy shows a clear quadratic trend, first decreasing before increasing again. Specifically,  the lowest accuracy happens at noise ratio of 0.6 and 0.8 in the case of the truncated normal noise example with $\mu=1$ and $\sigma=0.5$ (Fig.~\ref{subfig:normal_cn}), and the bimodal noise example with $\mu_1=2, \sigma_1=0.5, \mu_2=7, \sigma_2=5$ (Fig.~\ref{subfig:bimodal_cn}), respectively. The reason is that specific class examples with erroneous labels become more numerous than examples with the true class, e.g., more truck images are labelled as an automobile than automobile images. Such an effect is missing when testing on clean labels.

\begin{figure*}[t]
	\centering
	{\hfill
	\subfloat[Truncated normal 
	]{
	    \label{subfig:normal_cn}
	\includegraphics[width=0.30\linewidth ,height=2.65cm]{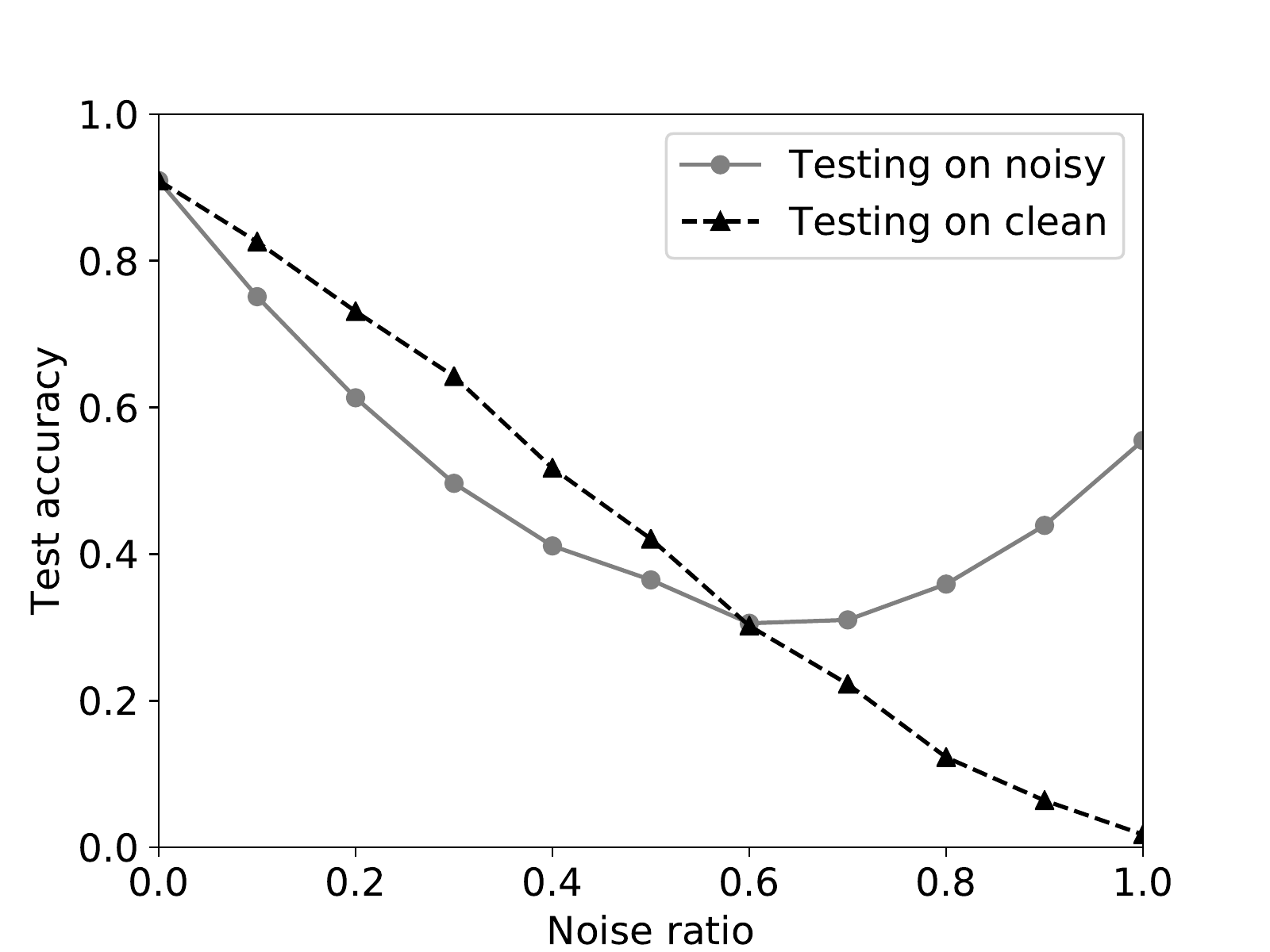}
	}\hfill
	  \subfloat[Bimodal 
	  ]{
	    \label{subfig:bimodal_cn}
	    \includegraphics[width=0.3\linewidth,height=2.65cm]{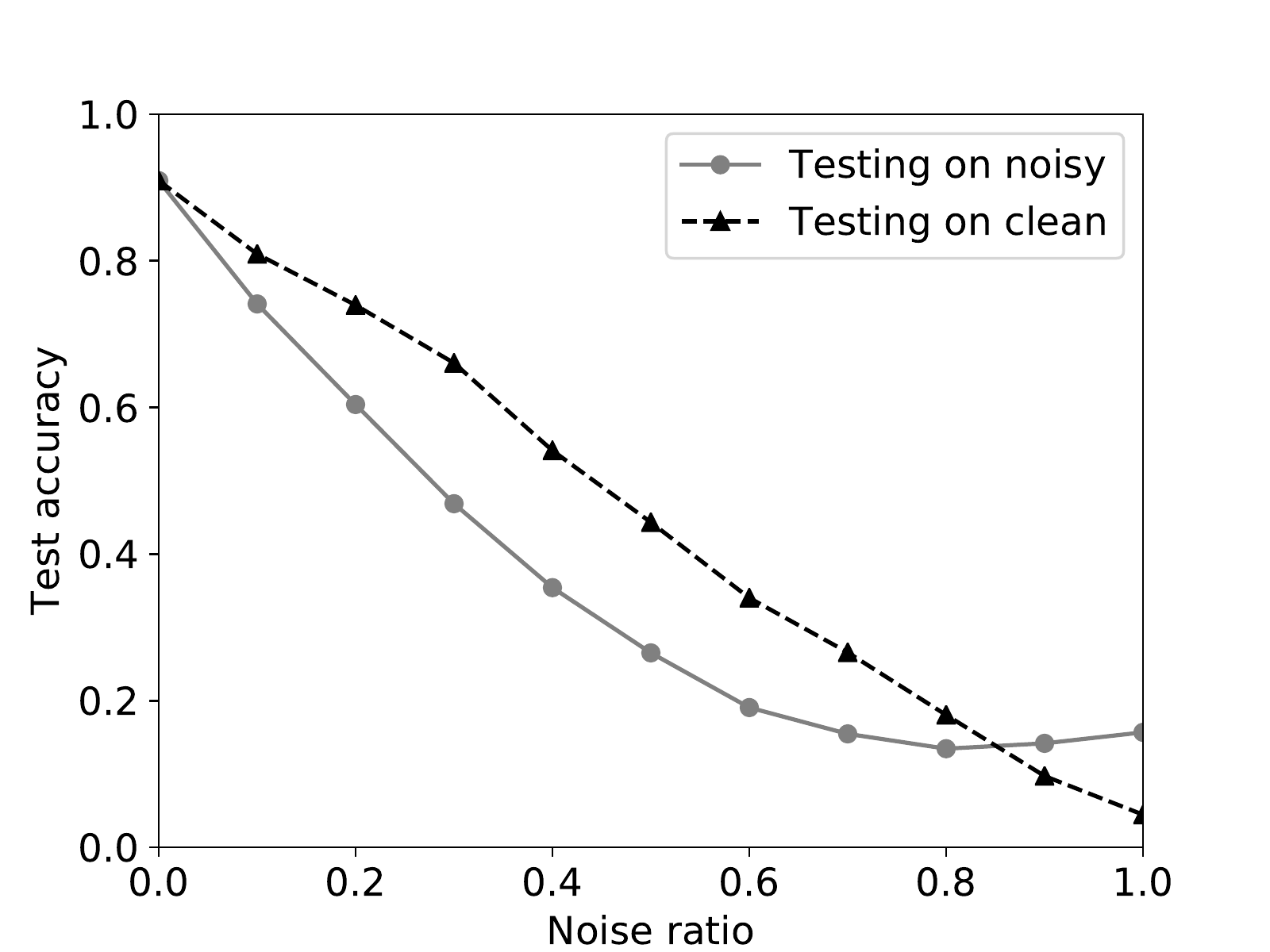}}
\hfill}
	\caption{Empirical testing on noisy and clean labeled data on CIFAR-10.}
	\label{fig:empcleannoise}
\end{figure*}

\section{Conclusion}
Motivated by the disparity of label noise patterns studied in the prior art, we first derive the analytical understanding of synthetic and real-world noise, i.e., how testing accuracy degrades with noise ratios and patterns. Challenging noise patterns identified here lead to the proposed learning framework, \TNet, which features light-weight adversarial learning and label noise resilient loss function. \TNet first adversely learns a noise transition matrix via a small set of trusted data and \system. Combining the estimated labels inferred from \system, \TNet computes a robust loss function from both given and inferred labels via dynamic weights according to the learning confidence, i.e., the entropy. We evaluate \TNet on CIFAR-10, CIFAR-100, and Clothing1M using a diverse set of synthetic and real-world noise patterns. The higher testing accuracy against state-of-the-art resilient networks shows that \TNet can effectively learn the noise transition and enhance the robustness of loss function against noisy labels.


%
%
%
\bibliographystyle{splncs04}
{\footnotesize
\bibliography{example_paper.bib}}
\end{document}